\documentclass{article}

\usepackage[utf8]{inputenc} 
\usepackage[T1]{fontenc}    
\usepackage{hyperref}       
\usepackage{url}            
\usepackage{booktabs}       
\usepackage{amsfonts}       
\usepackage{nicefrac}       
\usepackage{microtype}      
\usepackage{mathtools,amssymb}
\usepackage{microtype}
\usepackage{graphicx}
\usepackage{subfigure}
\usepackage{booktabs} 
\usepackage{algorithmic}
\usepackage{algorithm}

\usepackage{bbm}
\usepackage{amsmath}
\usepackage{amsthm}

\usepackage{times}
\usepackage{braket} 
\newtheorem{thm}{Theorem}

\newcommand{\norm}[1]{\left\lVert#1\right\rVert}
\newcommand{\algo}{DMEG}

\newcommand{\nr}{\mathbb{R}}

\newcommand{\nE}{\mathbb{E}}
\newcommand{\lamoptim}{\lambda^*}

\newcommand{\Strategy}{\mathcal{S}}
\newcommand{\strategy}{\mathbf{S}}
\newcommand{\cX}{\mathcal{X}} 
\newcommand{\cY}{\mathcal{Y}} 
\newcommand{\choS}{\mathcal{B}} 
\newcommand{\obs}{x} 
 
\newcommand{\cho}{b} 
\newcommand{\ml}{u} 
\newcommand{\cl}{c} 

\newcommand{\eqdef}{\triangleq}

\title{ Deep  Online Learning with Stochastic Constraints}

\author{
  Guy Uziel \\
  Department of Computer Science\\
  Technion - Israel Institute of Technology\\
}

\begin{document}

\maketitle
\begin{abstract} 
Deep learning models are considered to be state-of-the-art in many offline machine learning tasks. However, many of the techniques developed are not suitable for online learning tasks. The problem of using deep learning models with sequential data becomes even harder when several loss functions need to be considered simultaneously, as in many real-world applications.   
In this paper, we, therefore, propose a novel online deep learning training procedure which can be used regardless of the neural network's architecture, aiming to deal with the multiple objectives case.
We  demonstrate and show the effectiveness of our algorithm on the  Neyman-Pearson classification problem on several benchmark datasets.

\end{abstract}

\section{Introduction}

In many real-world applications, one has to consider the minimization of several loss functions simultaneously, which is, of course, an impossible mission. Therefore,  one objective is chosen as the primary function to minimize, leaving the others to be bound by pre-defined thresholds.  For example, in online portfolio selection \cite{BorodinE2005},  the ultimate goal is to maximize the wealth of the investor while keeping the risk bounded by a user-defined constant.  In the Neyman-Pearson (NP) classification  (see, e.g., \cite{RigolletT2011}), an extension of the classical binary classification,  the goal is to learn a classifier achieving low type-II error whose type-I error is kept below a given threshold. Another example is the online job scheduling in distributed data centers (see, e.g., \cite{hung2015scheduling}), in which a job router receives job tasks and schedules them to different servers to fulfill the service. Each server purchases power (within its capacity) from its zone market, used for serving the assigned jobs. Electricity market prices can vary significantly across time and zones, and the goal is to minimize the electricity cost subject to the constraint that incoming jobs must be served in time.

It is indeed possible to adjust any training algorithms capable of dealing with one objective loss to deal with multiple objectives by assigning a positive weight to each loss function. However, this modification turns out to be a  difficult problem, especially in the case where one has to maintain the constraints below a given threshold online.

Recently, several papers have put their focus on dealing with the multiple objective cases in the online setting. In the adversarial setting, it is known that multiple-objective is generally impossible when the constraints are unknown a-priory \cite{MannorTY2009}. In the stochastic setting,  \cite{MahdaviYJ2013}   proposed a framework for dealing with multiple objectives in the  i.i.d. case and \cite{uziel2017multi} has extended the above to the case where the underlying process is stationary and ergodic. 

The previous approaches, however, focused mainly on the online training of shallow models in this context.  Therefore, the main goal of this paper is to propose an online  training procedure capable of controlling several objectives simultaneously utilizing deep learning models, which  have witnessed tremendous  success  in a wide range of offline (batch) machine learning tasks applications \cite{lecun2015deep,bengio2013representation,he2016deep,krizhevsky2012imagenet}.

The training of deep neural networks is considered to be hard due to many challenges arising during the process.  For example,   vanishing  gradient,  diminishing  feature reuse \cite{he2016deep}, saddle  points,  local  minima  \cite{choromanska2015loss,dauphin2014identifying}, difficulties in choosing a good regularizer choosing hyperparameters.

Those challenges are even harder to tackle in the online regime, where the data is given sequentially, and because the data might not be stationary and exhibit different distributions at different periods (concept drift)\cite{gama2014survey}. Despite the challenges, having a way to learn a deep neural network in an online fashion can lead to more scalable,  memory efficient, and better performance of deep learning models in online learning tasks.

There are several approaches on how to train deep learning models in the online setting; first, we have the naïve approach: directly applying a standard Back-propagation training on a single instance at each online round. The main critical problem in this approach is the lack of ability to tune the model (e.g., choosing the right architecture and the hyper-parameters) in online-manner and in the case of composed loss to tune the weights on each of the losses. Since, in the online case, we remind, validation sets do not exist.

A different approach, which is used mainly in non-stationary data streams such as online portfolio selection, is to set two sliding windows, one for training and the other one for testing (see, e.g., \cite{jiang2017deep}).  After training the model and testing it, the sliding window is moved to the next period and trained all over again. It relays on the assumption that close data points should exhibit the same distribution.  This method suffers from a significant reduction of the training data, thus, pruned to overfitting and is less suitable for cases where real-time prediction.

Another approach exploits the principle of "shallow to deep."  This is based upon the observation that shallow models converge faster than deeper ones. This approach has been exploited in the offline batch setting, for example, using the function preservation principle \cite{goodfellow2014explaining,koh2017understanding}, and modifying the network architecture and objective functions,  e.g.,  Highway  Nets\cite{srivastava2015highway}. Those methods which exhibit improve convergence for the offline setting turned, as demonstrated by \cite{Sahoo2018},  not to be suitable for the online setting as the inference is made by the deepest layer which requires substantial time for convergence.

Deeply Supervised Nets \cite{lee2014deeply} are another way to implement the above principle. This method incorporates companion objectives at every layer, thus addressing the vanishing gradient and enabling the learning of more discriminative features at shallow layers. \cite{Sahoo2018} recently adjusted this approach to the online setting by setting weights on those layers and optimizing the weights in an online manner using the hedge algorithm.

In this paper, we follow the latter approach, as well. By modifying a given network's architecture, we modify the multi-objective problem into a convex one. Thus, allowing the use of the strong duality theorem. By doing so, we transform the problem into a minimax optimization problem by introducing the corresponding Lagrangian. Afterward, by a three-step optimization process which is responsible both for optimizing the primal and the dual variables and the weight's of the networks themselves, we can perform as well as the best layer in hindsight while keeping the constraints bounded below a given threshold. While we discuss here the case of only two objective functions, our approach can be extended easily to any arbitrary number of objective functions.

The paper is organized as follows: In Section~\ref{sec:formulation}, we define the multi-objective online optimization framework. In Section~\ref{sec:deep}, we present our deep training algorithm and prove that under mild conditions, it produces predictions for which the constraints hold. In Section~\ref{sec:empirical}, we demonstrate our algorithm on $3$ public benchmark datasets aiming to control the type-I error.

\section{Problem Formulation}
\label{sec:formulation}

We consider the following prediction game. At the beginning of each round, $t = 1, 2, \ldots$, the player receives  an observation $\obs_t \in \cX \subset \nr^d $ generated from an i.i.d. process.   The player is required to make a prediction $\cho_t\in \choS $, where $\choS \subset \nr ^m$ is a compact and convex set,  based  on past observations.   After making the prediction $\cho_t$, the label  $y_t \in \mathcal{Y}  \subset \nr^d$  is revealed and the player suffers  two losses, $\ml(\cho_t,y_t)$  and  $\cl(\cho_t,y_t)$, where $\ml$ and $\cl$ are  real-valued continuous functions and  convex w.r.t. their first argument. 


The  player is using a deep  neural network, $\strategy_w: \cX \rightarrow \choS$, parameterized with $w\in \nr^W$. We view  the player's prediction strategy as a sequence 
$\Strategy \eqdef \{\strategy_{w_t}\}^\infty_{t=1}$ of forecasting functions 
$\strategy_{w_t} : \cX \rightarrow \cY$;   that is,
the player's prediction  at round $t$ is given by $\strategy_{w_t}(\obs_t)$. 

 The player is required therefore to play the game with a strategy that minimizes the average $\ml$-loss,  $\frac{1}{T}\sum_{t=1}^{T}\ml(\strategy_{w_t}(x_t),y_t)$, while keeping the average $\cl$-loss $\frac{1}{T}\sum_{t=1}^{T}\cl(\strategy_{w_t}(x_t),y_t)$  bounded below a prescribed threshold $\gamma$.
 
 As in typical online learning problems \cite{helmboldSSW1998,CesaL2006}, the goal of the learner is to perform as good as $w^*$ satisfying the following  optimization problem:


\begin{equation}
\label{minprob}
\begin{aligned}
& \underset{w \in \nr^W}{\text{minimize}}
& & \nE \left[ \ml(\strategy_w(x),y) \right]
& \text{subject to}
& & \nE \left[ \cl(\strategy_w(x),y)\right] \leq \gamma, 
\end{aligned}
\end{equation}

Since the player is using a neural network, it is probable to   assume that Problem~\ref{minprob} is  not convex.  Therefore, as  in \cite{Sahoo2018,lee2015}    we modify  the network's architecture as follows: 
we denote the network's hidden layers  by $l_1,\ldots,l_L$ ($l_L$ is the output layer of $\strategy_w$),  and  attach an  output layer to each one of them, resulting in  $L$ predictions at each round, $\strategy_t(\obs_t) \eqdef \left(\strategy^1_{w_t}(\obs_t),\ldots ,\strategy^L_{w_t}(\obs_t )\right)$.


In other words, every output layer can be regarded as  an expert and can be seen as a strategy on its own. Therefore, by assigning a probability weight to the experts $p\in \Delta_L$, where $\Delta_L$ is the $L$-dimensional probability simplex we can modify the problem into the following one:\footnote{The construction above do not depend on the optimization  on the weights}

\begin{equation}
\label{minprob2}
\begin{aligned}
& \underset{p \in \Delta_L}{\text{minimize}}
& & \nE \left[ \ml( \langle p,\strategy_i(\obs_t) \rangle,y_t) \right]
& \text{subject to}
& & \nE \left[ \cl(\langle p,\strategy_t(\obs_t) \rangle,y_t)  \right] \leq \gamma.
\end{aligned}
\end{equation}   

Note that Problem~\ref{minprob2} is now a convex minimization problem over $\Delta_L$. However, there is no guarantee that a feasible solution even exists. In many real-world problems, one can come up with simple experts satisfying the constraints.
For example, in the Neyman-Pearson classification, one can use the strategy of always predicting the $0$-label class, resulting in $0$ type-I error, and thus satisfying any (reasonable) constraint on the type-I error. Another example is in the online portfolio selection, where adding a strategy of investing only in cash results in zero risk trading.  Thus, we assume that there exists such a simple expert, denoted by $\strategy^0$, satisfying the constraints and we artificially add him to the problem (resulting in that now the problem is minimized over $\Delta_{L+1}$). 

By the addition of this expert, we  require that the player minimizes the main loss, while the average violation of the constraint is bounded as follows:

\begin{gather*}
 \frac{1}{T}\sum_{t=1}^{T}\cl(\langle p_t,\strategy_t(\obs_t) \rangle,y_t)  \leq \gamma  + O(\frac{1}{T}),
\end{gather*}

Moreover, now Slater's condition holds, and the problem  is equivalent to finding  the saddle point of the Lagrangian function \cite{BenN2012}, namely,

\begin{gather*}
\min_{p \in \Delta_{L+1}}\max_{\lambda \in \nr^+}\mathcal{L}(p,\lambda),
\end{gather*}
where the Lagrangian is
\begin{gather*}
\mathcal{L}(p,\lambda)\eqdef \nE \left[ \ml(\langle p,\strategy_t(\obs_t) \rangle,y_t)  \right] +\lambda\left( \nE \left[ \cl(\langle p,\strategy_t(\obs_t) \rangle,y_t) -\gamma  \right)  \right] .
\end{gather*}
We denote the  optimal dual by $\lamoptim$. Moreover, we set a constant\footnote{This can be done, for example, by imposing some regularity conditions on the  objectives (see, e.g., \cite{MahdaviYJ2013}).} $\lambda_{\max}\geq 1$ such that $\lambda_{\max}>\lamoptim$, and set $\Lambda \eqdef [0,\lambda_{\max}]$.
We also define the \emph{instantaneous Lagrangian function}   as 
\begin{equation}
\label{eq:l_loss}
l(p,\lambda,y_t) \eqdef \ml(\langle p,\strategy_t(\obs_t) \rangle,y_t)+\lambda\left(\cl(\langle p,\strategy_t(\obs_t) \rangle,y_t)-\gamma\right).
\end{equation}

Summarizing the above, regardless of the neural network's architecture and on how it is trained, we were able to turn the optimization problem into a convex one, and by using the strong duality theorem, we turned the multi-objective problem into a minimax problem. 

In the next section, we present our algorithm, designed to jointly find the minimax point between the experts and optimizing the network's layers. 

\section{Deep Minimax Exponentiated Gradient}
\label{sec:deep}
\begin{algorithm}[tb]
	\caption{Deep Minimax Exponentiated Gradient (DMEG)  }
	\label{alg:algo}
	\begin{algorithmic}[1]
		\STATE {\bfseries Parameters:} $\eta,\eta^\lambda >0$:
		\STATE {\bfseries Initialize:} $\lambda_0 =0, p_0 = (\frac{1}{L+1},...,\frac{1}{L+1}) ,\hat{p}_0 = (\frac{1}{L+1},...,\frac{1}{L+1})  $.
		\FOR{ each period $t$ } 
				\STATE {\bf Receive}  context $x_{t}$
		\STATE {\bf Compute} predictions $\strategy_t(x_{t})$ of the experts
		\STATE {\bf Play} $(p_t,\lambda_t)$ 
		\STATE {\bf Suffer}  loss $l(\langle p_t, \strategy_t(x_{t}) \rangle,\lambda_t,y_t)$
 
		
		\STATE  {\bf Update  experts' weights:} 
		\begin{equation*}
		\mathbf{\hat{p}_{t+1,k}} \leftarrow \hat{p}_{t,k} \exp \left(-\eta \sum_{i=1}^t \nabla_k l(\langle p_t, \strategy_t(x_{t}) \rangle,\lambda_i, y_i) \right) \quad k = 1,\ldots,L+1
		\end{equation*}
		\begin{equation*}
		\mathbf{p}_{t+1} \leftarrow  \frac{\hat{p}_{t+1}}{\norm{\hat{p}_{t+1}}_1}
		\end{equation*}
		\STATE  {\bf Update $\lambda$:} 
		\begin{equation*}
		\mathbf{\lambda}_{t+1} \leftarrow \lambda_{max}\frac{\exp \left(\eta^\lambda \sum_{i=1}^t \nabla_{\lambda} l(\langle p_t, \strategy_t(x_{t}) \rangle,\lambda_i, y_i) \right)}{1+ \exp \left(\eta^\lambda \sum_{i=1}^t \nabla_{\lambda}l(\langle p_t, \strategy_t(x_{t}) \rangle,\lambda_i, y_t)\right)}
		\end{equation*}
		\STATE  {\bf Update $\mathbf{w}_{t+1}$:} 
		\begin{equation*}
		\mathbf{w}_{t+1} \leftarrow Backprop(l(\langle p_t, \strategy_t(x_{t}) \rangle,\lambda_t,y_t),\mathbf{w}_{t}) 
		\end{equation*}
		\ENDFOR
		
	\end{algorithmic}
\end{algorithm}

We now turn to present our algorithm Deep Minimax Exponentiated Gradient (DMEG), designed for jointly optimizing the network's parameters, $w$, and tuning the minimax weights between the different experts, $p$ and $\lambda$. The algorithm is outlined  at Algorithm~\ref{alg:algo}; at the beginning of each round the algorithm receives an observation, and the corresponding predictions of the different experts and predicts $p_t,\lambda_t$ (lines 4-5). Afterwards, the label $y_t$ is revealed and  the algorithm  suffers the instanious Lagrangian loss   (lines 6-7). Then after receiving the loss, the algorithm  preforms a three-step optimization where the first and second step (lines 8 and 9 respectively) are in place to ensure converges to the minimax solution of Problem~\ref{minprob2}. The last step is the backpropagation step whose role is to update the weights of the neural network, and thus improving the performance of the experts (line 10). The optimization steps are described in detail below:

\paragraph{Updating the experts' weights}
Using the  "shallow to deeper" observation discussed before, the shallow experts may converge faster and exhibit better results than the deeper layers at early stages, and at later stages, deeper experts might exhibit the better results.  Thus, we aggregate the experts' prediction using the well-known expert learning algorithm Exponentiated Gradients (EG) \cite{helmboldSSW1998}, this algorithm punishes the weights of the experts in a proportional way to the loss they incur during the training period.

\paragraph{Updating $\lambda$}
Since we are looking for a minimax solution, we need to update the lambda parameter alternately, and after the update of $p$, we update $lambda$ as well. The updates are in place to ensure the maximization of the Lagrangian. This is done  using again the EG algorithm applied over two static experts, one predicting $0$ and the other one predicting $\lambda_{max}$.

\paragraph{Updating the neural network}

Independently of the two above stages, we optimize the prediction of the neural network based on previous weights and lambda, and thus, we perform a back propagation step using the weighted Lagrangian. The optimization can be done using any traditional deep learning optimizer.

The prediction of the network is a weighted average of all the predictions of the different experts $\langle p_t, \strategy_t(x_{t}) \rangle$. Therefore, the total loss of the network which is used for the backpropagation step is given by 
\begin{gather*}
    l(\langle p_t, \strategy_t(x_{t}) \rangle,\lambda_t,y_t).
\end{gather*}

The advantages of such integrated optimization are threefold: first, as we prove in the next subsection, it ensures that the predictions of the network satisfy the constraints. Second, the EG algorithm between the experts tunes the model towards the best performing expert and third by the unique structure of the network the layers share information between them.

\subsection{Theoretical guarantee}
In this subsection, we state and prove that our procedure which aggregates between the different predictors ensures that the predictions generated by $\algo$   the violations of the constraints will be bounded. We now state and prove our results; we note that we do not assume anything regarding how $(x_t,y_t)$ are generated. Since we always can scale a bounded loss function, we will assume that $l$ takes values in $[0,1]$  
 \begin{thm}
 	Let $\cho_1,\ldots,\cho_T$, where $b_t = \langle p_t, \strategy_t(x_{t}) \rangle$   be the predictions generated by $\algo$ when applied over an arbitrary neural network $\strategy_w$. Denote $G_1$,$G_2>0$ to be bounds on the derivatives w.r.t. the first and the second argument of $l$ respectively, then if we set $\eta = \frac{1}{G_1}\sqrt{\frac{\log{L+1}}{T}}$ and $\eta^\lambda = \frac{1}{G_2}\sqrt{\frac{\log{2}}{T}}$ then  the following holds:
 	\begin{gather}
 	\label{eq:G1}
	\frac{1}{T}\sum_{t=1}^{T}\cl(\cho_t,y_t)   \leq \gamma +  4\max(G_1,G_2)\sqrt{\frac{\log{(L+1)}}{T}}.
 	\end{gather}
 \end{thm}

\begin{proof}
	
	Using the guarantees of the EG algorithm \cite{helmboldSSW1998} we get that for every  $p \in \Delta_{L+1}$: 
	\begin{gather}
	\frac{1}{T}\sum_{t=1}^{T} l(\cho_t,\lambda_t,y_t) - \frac{1}{T}\sum_{t=1}^{T} l(\langle p, \strategy_t(x_{t}) \rangle,\lambda_t,y_t) \leq 2G_1\sqrt{\frac{\log{(L+1)}}{T}}.  \label{eq:p1}
	\end{gather}
	Using the guarantees of EG applied over the $lambda$ updates we get that for every $\lambda \in \Lambda$ 
		\begin{gather}
		\frac{1}{T}\sum_{t=1}^{T} l(\cho_t,\lambda,y_t) - \frac{1}{T}\sum_{t=1}^{T} l(\cho_t,\lambda_t,y_t) \leq 2G_2\sqrt{\frac{\log{(2)}}{T}}. \label{eq:p2}
	\end{gather}

Summing Equations \eqref{eq:p1}+\eqref{eq:p2} we get that for every choice of $(p,\lambda) \in \Delta_{L+1} X \Lambda$.
		\begin{gather}
		\frac{1}{T}\sum_{t=1}^{T} l(\cho_t,\lambda,y_t) - \frac{1}{T}\sum_{t=1}^{T} l(\langle p, \strategy_t(x_{t}) \rangle,\lambda_t,y_t)\leq 4\max(G_1,G_2)\sqrt{\frac{\log{(L+1)}}{T}}.\label{eq:psum}
\end{gather}

Now, If we set $\lambda = 1$, and choose $p$ such that all the probability mass is on  the artificial expert then by using again  Equation~\eqref{eq:psum}  we get the following  
\begin{gather}
 \frac{1}{T}\sum_{i=1}^{T} \cl(\cho_i,x_i)- \frac{1}{T}\sum_{i=1}^{T}\cl(\langle p, \strategy_t(x_{t}) \rangle,x_i) \leq   4\max(G_1,G_2)\sqrt{\frac{\log{(L+1)}}{T}},
\end{gather}
using the guarantees of the artificial expert we get that  Equation~\eqref{eq:G1} holds.
\end{proof}

\section{Empirical Results}
\label{sec:empirical}

In this section, we demonstrate the ability of $\algo$ to online learn deep neural networks while controlling a given constraint; we experiment with our algorithm on the Neyman-Pearson classification defined below.

\subsection{Neyman-Pearson classification}
In the classical binary classification, the learner's goal is to minimize the classification error. This measure does not distinguish between type-I and type-II errors.  The Neyman-Pearson (NP) paradigm, is an extension of the classical binary classification in the sense that now the learner minimizes the type-II error while upper bounding the type-I  error, in other words, the NP classification is suitable for the cases when one error is more pricey than the other.  For instance, failing to detect a  tumor has far more severe consequences than wrongly classifying a healthy area; other examples include spam filtering, machine monitoring.
Therefore,  it is needed to put the priority on controlling the false negative rate, and for a given bound $\alpha$ on the type-I error. For a classifier $h\in\mathcal{H}$, whose (discrete)  predictions are denoted by $\hat{h}(x)$ we define the NP problem as follows:
\begin{equation}
\label{minprob3}
\begin{aligned}
& \underset{h\in\mathcal{H}}{\text{minimize}}
& & \nE\left[\hat{h}(x)  \neq y \mid y=0\right] 
& \text{subject to}
& & \nE\left[\hat{h}(x) \neq y \mid y=1\right]  \leq \gamma, 
\end{aligned}
\end{equation} 

As in classical classification, one approach  to solve the Neyman-Pearson classification is by using convex surrogate losses (see, e.g., \cite{rigollet2011neyman}). Therefore, for a convex surrogate loss function $R(\cdot,\cdot)$ we get the following convex problem: 

\begin{equation}
\label{minprob3}
\begin{aligned}
& \underset{h\in\mathcal{H}}{\text{minimize}}
& & \nE\left[ R(h(x),y) \mid y=0\right] 
& \text{subject to}
& & \nE\left[R(h(x),y) \mid y=1\right]  \leq \gamma, 
\end{aligned}
\end{equation} 

During our experiments we used the binary cross entropy to serve as the surrogate loss function.

We used the following binary classification datasets, all are  available in the public domain.


\paragraph{Susy}
Susy problem actualizes a big streaming data problem consisting of $5$ million records. This dataset was produced by Monte-Carlo simulations and is the classification of signal processes generating super-symmetric particles \cite{baldi2014searching}. 
Each instance is represented by  $18$ input attributes where the first eight input attributes present the kinematic properties while the remainders are simply the function of the first eight input features. 

\paragraph{Higgs}

This classification problem aims to distinguish between a signal process which produces Higgs bosons and a background process which does not. The features contain kinematic properties measured by the particle detectors in the accelerator and high-level features derived by physicists to help discriminate between the two classes \cite{baldi2014searching}.

\paragraph{CD}
The concept drift dataset \cite{Sahoo2018}  contains three different segments, each comprising a third of the data. All the three segments were generated from an 8-hidden layer network; this challenging concept drift dataset is a part of a broader machine learning task aiming to deal with concept drift.

The properties of these public datasets are summarized in Table~\ref{tab:dataset}.

\begin{table}[t]
	\caption{Properties of the datasets}
	\label{tab:dataset}
	\vskip 0.15in
	\begin{center}
		\begin{sc}
			\begin{tabular}{l||lll}
				\hline
				
				
				Dataset & Length & Features & Type    \\
				\midrule

				SUSY     & $5$M & $28$ & Stationary \\        
				
				Higgs     & $5$M & $18$    &Stationary  \\

				CD     & $4.5$M & $50$   & Concept drift  \\
				
				\hline
			\end{tabular}
		\end{sc}
	\end{center}
	\vskip -0.1in
\end{table}

\begin{figure}[!htb]

    \caption{Trade-off between the type-II (Y-axis) error and type-I (X-axis)}
    \label{fig:con}
	\minipage{0.32\textwidth}
	\centering
	\includegraphics[width=\linewidth,height=1.6in]{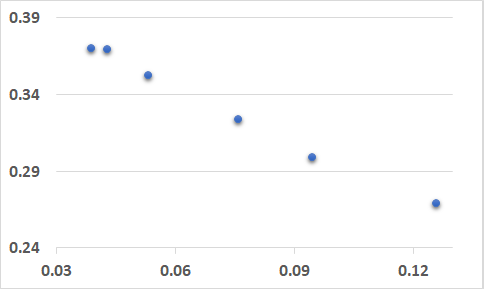}
    \vspace{\abovecaptionskip}%
    CD
	\endminipage\hfill
	\minipage{0.32\textwidth}
	\centering
	\includegraphics[width=\linewidth,height=1.6in]{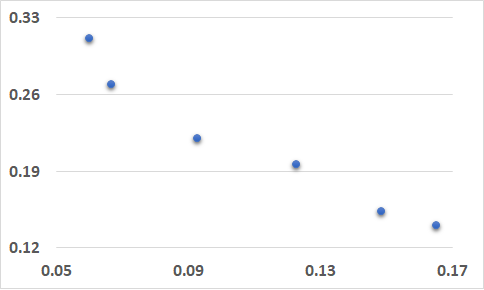}
    \vspace{\abovecaptionskip}%
    Susy
	\endminipage\hfill
	\minipage{0.32\textwidth}%
	\centering
	\includegraphics[width=\linewidth,height=1.6in]{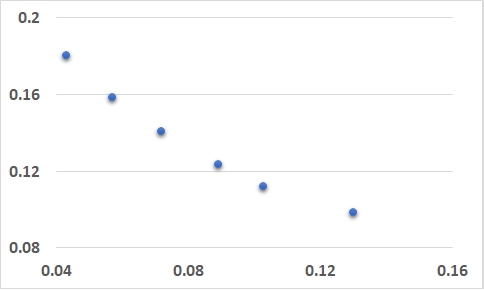}
    \vspace{\abovecaptionskip}%
    Higgs
	\endminipage
\end{figure}

\begin{figure}[!htb]
    \caption{The average constraint for $\algo$ with $\gamma=0.21$}
    \label{fig:bound}
	\minipage{0.32\textwidth}
	\centering
	\includegraphics[width=\linewidth,height=1.6in]{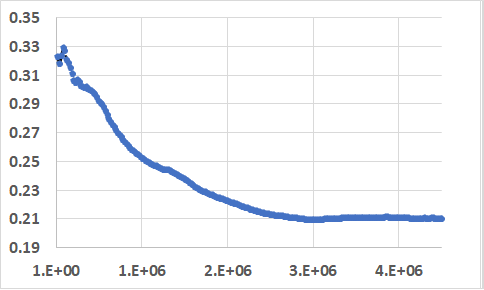}
    \vspace{\abovecaptionskip}%
    CD
	\endminipage\hfill
	\minipage{0.32\textwidth}
	\centering
	\includegraphics[width=\linewidth,height=1.6in]{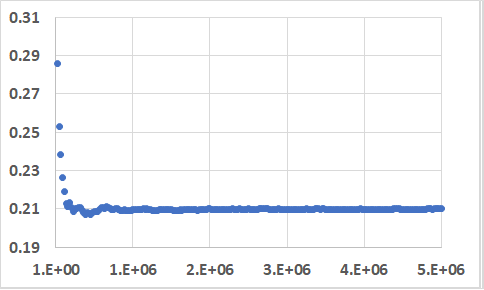}
    \vspace{\abovecaptionskip}%
    Susy
	\endminipage\hfill
	\minipage{0.32\textwidth}%
	\centering
	\includegraphics[width=\linewidth,height=1.6in]{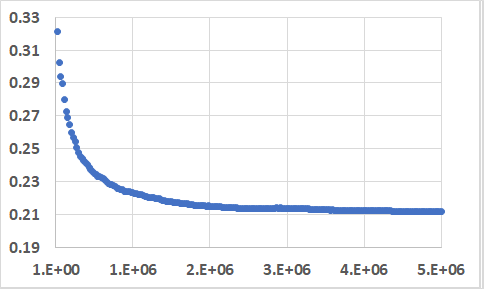}
    \vspace{\abovecaptionskip}%
    Higgs
	\endminipage
\end{figure}

\begin{table*}[htb!]

	\centering
	\caption{ Type-I and the Type-II (in parentheses) of $\algo$ with $\gamma=0.2$} \label{Table:Empirical-Results}
	

		\begin{sc}
			\setlength\tabcolsep{2.5pt}
			\begin{tabular}{l||lll||l l}
				\hline

				
				 & $\phantom{ff}$ BL & $\phantom{ff}$HBP  & $\phantom{ff}$MOL & $\phantom{ff}$BE  &   $\algo$ \\
				\midrule

				 SUSY         & 
				$.121$ $(.09)$  &$.11$ $(.093)$ & $.081$  $(.163)$&  $.071$ $(.136)$& $.073$ $(.14)$ \\        
				  HIGGS         & $.172$ $(.155)$ & $.161$ $(.152)$ & $.097$ $(.326)$ &  $.090$ $(.215)$ & $.091$ $(.218)$\\
				   CD           & $.217$ $(.208)$ & $.194$ $(.219)$ & $.096$ $(.370)$ & $.051$  $(.352)$ & $.052$  $(.353)$ \\

				\midrule

			\end{tabular}
			
		\end{sc}

	
\end{table*}

\subsection{Implementation and results}

To apply $\text{\algo}$ strategy,  we used a $20$ layer DNN, as in \cite{Sahoo2018}, with $100$ units in each hidden layer and with ReLU Activation. For each one of the $19$ hidden layers layer a clasifiier (a fully-connected layer followed by softmax activation) was attached, resulting in $19$ classifiers (experts), each with depth ranging from $2,\ldots,20$. The main objective of our experiments is to examine how well $\text{\algo}$  maintains the type-$I$ error constrains. The second objective is to examine the ability of $\text{\algo}$ to track the best expert.  

The inclusion of the artificial expert is rather for theoretical purposes, in order to ensure that Slater's condition holds for Problem~\ref{minprob2}. In our experiments, however,  we did not add this expert at all, and as we present later on we were still able to fulfill the constraints.

For the implementation of $\text{\algo}$, we set $\eta=0.01$ and $\eta^\lambda=0.01$ without prior tuning. The network was trained using a fixed learning rate during all the training rounds and across all the experiments and set to $0.001$, the optimization was done using Nesterov accelerated gradient. 

We used the following baselines:
\begin{itemize}

    \item  BL - the best result obtained by running instances of a fully connected DNNs with  layers ranging from (2,3,4,8,16) and with the same parameters as ours, all were trained using a learning rate of $0.01$ Nesterov accelerated gradient. 
    \item   HBP - an algorithm designed for the deep online learning setting, with the same architecture and hyper-parameters as described in \cite{Sahoo2018}.
    \item MOL - an implementation of a shallow online multi-objective model with parameters $\eta=0.01,\eta^\lambda=0.01$ \cite{MahdaviYJ2013}.
    \item BE - the best performing expert of our model. The expert with the lowest Type-II error among the experts who satisfied the constraints.   
\end{itemize}

All the experiments were implemented using  Keras \cite{chollet2015keras}.

In Table~\ref{Table:Empirical-Results}, we present the type-I error and the type-II error (in parentheses) of the different algorithms where for $\algo$ we set $\gamma = 0.2$.  It is the place to note that since we have convexified the NP classification problem, we expect to get  type-I error lower than $\gamma$. First, as we can observe,  all the existing methods of training deep neural networks online, training algorithms cannot take into consideration the constraint. Second, we can observe that our algorithm, as expected, can maintain the constraint across all the datasets. Moreover, by comparing the results of $BE$, the best expert of our model, we can see that we performed nearly as the best expert both in terms of type-I error and type-II error, emphasizing the ability of our algorithm can track the best performing layer. Together with the superior performance over $BL$, proves the usefulness of $\algo$ in tuning the model for the problem. The inferior performance of MOL comparing to $\algo$, manifesting the need in training deep online models with constraints. 

In another set of experiments, we checked how well our procedure could control the Type-I error, when $\gamma$ is changed. 
 Therefore for each dataset we run $6$ instances of $\algo$ with different values of $\gamma \in \{0.15,0.18,0.21,0.24,0.27,0.3\}$.  The results of this are presented at  Figure~\ref{fig:con}. First, we can observe that $\algo$ successfully maintained the constraints. Second,  we observe the inevitable tradeoff between the type-I and the Type-II error, which forms the shape of a concave Pareto-frontier. Table~\ref{tab:cvar} shows the average type-I error of all those instances across all the datasets and over different duration in  the training period.

Figure~\ref{fig:bound} presents the average constraint function across the $3$ datasets, for $\gamma=0.21$. As we can see, our algorithm well maintains this constraint across all the datasets during the training period.

\begin{table}[t]
	\caption{Average type-I error of $\algo$ with different values of $\gamma$ and at different stages}
	\label{tab:cvar}
	\vskip 0.15in
	\begin{center}
		\begin{small}
			
			\begin{sc}
				\begin{tabular}{l| l||llllll}
					\hline
					
					
					$\phantom{aa}$ Stage &Dataset & $\algo_{.15}$ & $\algo_{.18}$ & $\algo_{.21}$  & $\algo_{.24}$ &  $\algo_{.27}$ & $\algo_{.3}$     \\
					\midrule

					&SUSY     & $0.044$ & $0.057$  & $0.069$ & $0.087$ & $0.101$ & $0.122$  \\        
					
					$ 0\%-25\%$ &HIGGS    & $0.072$ & $0.094$  & $0.097$ & $0.126$ & $0.151$ & $0.178$  \\
					&CD     & $0.051$ & $0.059$  & $0.071$ & $0.105$ & $0.138$ & $0.182$  \\ 
										\midrule

					&SUSY     & $0.043$ & $0.058$  & $0.073$ & $0.092$ & $0.105$ & $0.127$  \\        
					
					$ 25\%-50\%$ &HIGGS & $0.066$ & $0.080$  & $0.087$ & $0.119$ & $0.144$ & $0.162$    \\
					&CD     & $0.055$ & $0.058$  & $0.072$ & $0.103$ & $0.127$ & $0.153$  \\ 
										\midrule

					&SUSY     & $0.042$ & $0.056$  & $0.072$ & $0.090$ & $0.102$ & $0.125$  \\        
					
					$ 50\%-75\%$& HIGGS     & $0.062$ & $0.78$  & $0.085$ & $0.104$ & $0.141$ & $0.177$  \\
					&CD     & $0.048$ & $0.053$  & $0.068$ & $0.086$ & $0.106$ & $0.129$ \\ 
										\midrule

					&SUSY     & $0.044$ & $0.058$  & $0.071$ & $0.088$ & $0.101$ & $0.123$  \\        
					
					$ 75\%-100\%$&HIGGS     & $0.064$ & $0.079$  & $0.096$ & $0.124$ & $0.150$ & $0.172$  \\
					&CD     & $0.030$ & $0.052$  & $0.067$ & $0.079$ & $0.103$ & $0.118$ \\ 
				\end{tabular}
			\end{sc}
		\end{small}
		
	\end{center}
	\vskip -0.1in
\end{table}

\section{conclusions}

Training and utilizing deep neural networks in an online learning setting is a challenging task, especially with the need to consider multiple objectives simultaneously. Therefore in this paper, we presented $\algo$, a novel approach to training a neural network on the fly while considering several objectives. Due to the non-convex nature of deep neural networks, we modified the problem in order to utilize the strong duality theorem and the Lagrangian relaxation. We also proved and demonstrated that our algorithm is capable of controlling given constraints on several datasets.

For future research, we wish to investigate further ways to train a neural network in the online setting. By bridging the gap between the approach, we took in this paper and training on the existing dataset (as done in the sliding window approach). On the one hand, the on-the-fly approach presented here gives us the ability the train models in an efficient way but in the cost of not fully optimizing the network and utilizing the data compare to the first approach. Therefore we wish to investigate whether we can devise ways to better trade-off between the two.

\bibliography{Bibliography}

\begin{thebibliography}{10}

\bibitem{krizhevsky2012imagenet}
Alex A.~Krizhevsky, I.~Sutskever, and G.~Hinton.
\newblock Imagenet classification with deep convolutional neural networks.
\newblock In {\em Advances in neural information processing systems}, pages
  1097--1105, 2012.

\bibitem{baldi2014searching}
P.~Baldi, P.~Sadowski, and D.~Whiteson.
\newblock Searching for exotic particles in high-energy physics with deep
  learning.
\newblock {\em Nature communications}, 5:4308, 2014.

\bibitem{BenN2012}
A.~Ben-Tal and A.~Nemirovsky.
\newblock Optimization iii.
\newblock {\em Lecture Notes}, 2012.

\bibitem{bengio2013representation}
Y.~Bengio, A.~Courville, and P.~Vincent.
\newblock Representation learning: A review and new perspectives.
\newblock {\em IEEE transactions on pattern analysis and machine intelligence},
  35(8):1798--1828, 2013.

\bibitem{BorodinE2005}
A.~Borodin and R.~El-Yaniv.
\newblock {\em Online Computation and Competitive Analysis}.
\newblock Cambridge University Press, 2005.

\bibitem{CesaL2006}
N.~Cesa-Bianchi and G.~Lugosi.
\newblock {\em Prediction, Learning, and Games}.
\newblock Cambridge University Press, 2006.

\bibitem{chollet2015keras}
F.~Chollet et~al.
\newblock Keras, 2015.

\bibitem{choromanska2015loss}
A.~Choromanska, M.~Henaff, M.~Mathieu, G.~Arous, and Y.~LeCun.
\newblock The loss surfaces of multilayer networks.
\newblock In {\em Artificial Intelligence and Statistics}, pages 192--204,
  2015.

\bibitem{dauphin2014identifying}
Y.~Dauphin, R.~Pascanu, C.~Gulcehre, K.~Cho, S.~Ganguli, and Y.~Bengio.
\newblock Identifying and attacking the saddle point problem in
  high-dimensional non-convex optimization.
\newblock In {\em Advances in neural information processing systems}, pages
  2933--2941, 2014.

\bibitem{gama2014survey}
J.~Gama, I.~{\v{Z}}liobait{\.e}, A.~Bifet, M.~Pechenizkiy, and A.~Bouchachia.
\newblock A survey on concept drift adaptation.
\newblock {\em ACM computing surveys (CSUR)}, 46(4):44, 2014.

\bibitem{goodfellow2014explaining}
I.~Goodfellow, J.~Shlens, and C.~Szegedy.
\newblock Explaining and harnessing adversarial examples.
\newblock {\em arXiv preprint arXiv:1412.6572}, 2014.

\bibitem{he2016deep}
K~. He, X.~Zhang, S.~Ren, and J.~Sun.
\newblock Deep residual learning for image recognition.
\newblock In {\em Proceedings of the IEEE conference on computer vision and
  pattern recognition}, pages 770--778, 2016.

\bibitem{helmboldSSW1998}
D.P. Helmbold, R.E. Schapire, Y.~Singer, and M.K. Warmuth.
\newblock On-line portfolio selection using multiplicative updates.
\newblock {\em Mathematical Finance}, 8(4):325--347, 1998.

\bibitem{hung2015scheduling}
C.~Hung, L.~Golubchik, and M.~Yu.
\newblock Scheduling jobs across geo-distributed datacenters.
\newblock In {\em Proceedings of the Sixth ACM Symposium on Cloud Computing},
  pages 111--124. ACM, 2015.

\bibitem{jiang2017deep}
Z.~Jiang, D.~Xu, and J.~Liang.
\newblock A deep reinforcement learning framework for the financial portfolio
  management problem.
\newblock {\em arXiv preprint arXiv:1706.10059}, 2017.

\bibitem{koh2017understanding}
P.~Koh and P.~Liang.
\newblock Understanding black-box predictions via influence functions.
\newblock In {\em Proceedings of the 34th International Conference on Machine
  Learning-Volume 70}, pages 1885--1894. JMLR. org, 2017.

\bibitem{lecun2015deep}
Y.~LeCun, Y.~Bengio, and G.~Hinton.
\newblock Deep learning.
\newblock {\em nature}, 521(7553):436, 2015.

\bibitem{lee2014deeply}
C.~Lee, S.~Xie, P.~Gallagher, Z.~Zhang, and Z.~Tu.
\newblock Deeply-supervised nets.
\newblock {\em arXiv preprint arXiv:1409.5185}, 2014.

\bibitem{lee2015}
C.~Lee, S.~Xie, P.~Gallagher, Z.~Zhang, and Z.~Tu.
\newblock Deeply-supervised nets.
\newblock In {\em Artificial Intelligence and Statistics}, pages 562--570,
  2015.

\bibitem{MahdaviYJ2013}
M.~Mahdavi, T.~Yang, and R.~Jin.
\newblock Stochastic convex optimization with multiple objectives.
\newblock In {\em Advances in Neural Information Processing Systems}, pages
  1115--1123, 2013.

\bibitem{MannorTY2009}
S.~Mannor, J.~Tsitsiklis, and J.~Yu.
\newblock Online learning with sample path constraints.
\newblock {\em Journal of Machine Learning Research}, 10(Mar):569--590, 2009.

\bibitem{RigolletT2011}
P.~Rigollet and X.~Tong.
\newblock Neyman-pearson classification, convexity and stochastic constraints.
\newblock {\em Journal of Machine Learning Research}, 12(Oct):2831--2855, 2011.

\bibitem{rigollet2011neyman}
P.~Rigollet and X.~Tong.
\newblock Neyman-pearson classification, convexity and stochastic constraints.
\newblock {\em Journal of Machine Learning Research}, 12(Oct):2831--2855, 2011.

\bibitem{Sahoo2018}
D.~Sahoo, Q.~Pham, J.~Lu, and S.~Hoi.
\newblock Online deep learning: learning deep neural networks on the fly.
\newblock In {\em Proceedings of the 27th International Joint Conference on
  Artificial Intelligence}, pages 2660--2666. AAAI Press, 2018.

\bibitem{srivastava2015highway}
R.~Srivastava, K.~Greff, and J.~Schmidhuber.
\newblock Highway networks.
\newblock {\em arXiv preprint arXiv:1505.00387}, 2015.

\bibitem{uziel2017multi}
G.~Uziel and R.~El-Yaniv.
\newblock Multi-objective non-parametric sequential prediction.
\newblock In {\em Advances in Neural Information Processing Systems}, pages
  3372--3380, 2017.

\end{thebibliography}
\bibliographystyle{plain}
\end{document}